\documentclass[letterpaper, 10 pt, conference]{ieeeconf}

\IEEEoverridecommandlockouts 

\usepackage{stackrel}
\usepackage{graphicx}
\usepackage{siunitx}
\graphicspath{{./Images/}}
\usepackage{amssymb,amsmath}
\usepackage{amsthm}
\usepackage{multicol}
\usepackage{xcolor}
\usepackage[noadjust]{cite}

\usepackage{bm} 
\newtheorem{remark}{Remark}
\newtheorem{definition}{Definition}
\newtheorem{lemma}{Lemma}
\newtheorem{theorem}{Theorem}

\newtheorem{assumption}{Assumption}

\title{\LARGE \bf Distributed Bandits: Probabilistic Communication on $d$-regular Graphs}

\author{Udari Madhushani and Naomi Ehrich Leonard
\thanks{This research has been supported in part by ONR grants N00014-18-1-2873 and N00014-19-1-2556 and ARO grant W911NF-18-1-0325.}
\thanks{U. Madhushani and N.E. Leonard are with Department of Mechanical and Aerospace Engineering, Princeton University, Princeton, NJ 08544, USA.
        {\tt\small \{udarim, naomi\}@princeton.edu}}%
}

\begin{document}

\maketitle
\thispagestyle{empty}
\pagestyle{empty}

\begin{abstract}
We study the decentralized multi-agent multi-armed bandit problem for agents that communicate with probability over a network defined by a $d$-regular graph. Every edge in the graph has probabilistic weight $p$ to account for the ($1\!-\!p$) probability  of a communication link failure. At each time step, each agent chooses an arm and receives a numerical reward associated with the chosen arm. After each choice, each agent observes the last obtained reward of each of its neighbors with probability $p$. We propose a new Upper Confidence Bound (UCB) based algorithm and analyze how agent-based strategies contribute to minimizing group regret in this probabilistic communication setting. We provide theoretical guarantees that our algorithm outperforms state-of-the-art algorithms. We illustrate our results and validate the theoretical claims using numerical simulations.
\end{abstract}

\allowdisplaybreaks

\section{Introduction} \label{sect:Introduction}
Distributed learning benefits individual and group decision making in both natural and engineered systems. Communication constraints and security concerns often prevent the possibility of centralized information gathering and centralized decision making. This necessitates the development of effective decentralized communication protocols and decision-making strategies.  
Decentralized strategies allow individual agents to make decisions using locally available information, i.e., information that they observe or is communicated to them from their neighbors, as defined by a communication network. As a result,  decision-making performance 
is strongly affected by the structure and the reliability of the network. 

Consider a group of agents making sequential decisions in an uncertain environment. Each agent is faced with the same problem of repeatedly choosing an option from a given fixed set of options. After every choice, each agent receives a numerical reward drawn from a probability distribution associated with the chosen option. The objective of each agent is to maximize its individual cumulative reward while contributing to maximizing the group cumulative reward.
The best strategy in this situation is to repeatedly choose the optimal option, i.e., the option that provides the maximum average reward. However, agents are unaware of the expected reward values of the options. Each individual is required to exploit the options that are known to provide high rewards and explore the lesser known options in order to identify the options that might potentially provide higher rewards. This is known as the explore-exploit dilemma in decision theory. In a networked setting, agents can learn the reward distributions with fewer explorations by sharing information.

The multi-armed bandit (MAB) problem is a mathematical model developed  to capture the salient features of the explore-exploit trade-off \cite{robbins1952some,lai1985asymptotically,lattimore2020bandit}. Initially, MAB was defined for a single agent making choices among multiple arms, equivalent to options, and receiving a reward after every choice.  
In the stochastic MAB problem, rewards are drawn from probability distributions. In the non-stochastic MAB problem, each option is assigned a deterministic sequence of rewards. In both problem settings, the goal is to maximize the expected cumulative reward. This is equivalent to minimizing the expected cumulative regret defined as the expected loss incurred by the agent through sampling sub-optimal options. It is well established that for the stochastic MAB the cumulative regret grows as $\Theta(\log T)$, \cite{auer2002finite}  and for the non-stochastic MAB the cumulative regret grows as $\Theta(\sqrt{T})$, \cite{bubeck2012regret} where $T$ is the time horizon of the decision-making process. In this paper we focus on stochastic MAB problems. The paper \cite{lai1985asymptotically} proposed a family of algorithms to achieve a logarithmic upper bound for cumulative regret asymptotically in $T$. A lower bound that matches the order of the upper bound was also proposed, establishing the result that any asymptotically efficient algorithm samples suboptimal options at least logarithmically in $T$. 
The paper \cite{auer2002finite} proposed a family of Upper Confidence Bound (UCB) algorithms to achieve logarithmic cumulative regret uniformly in $T$.

The papers \cite{anantharam1987asymptotically,martinez2019decentralized,landgren2020distributed,kolla2018collaborative,landgren2018social,wang2020optimal} studied distributed bandit problems with static communication protocols. A centralized multi-agent setting was considered in \cite{anantharam1987asymptotically} and a decentralized multi-agent setting was considered in  \cite{martinez2019decentralized,landgren2020distributed} in which running consensus algorithms are used to update estimates and provide graph-structure-dependent performance measures. The papers \cite{kolla2018collaborative,landgren2018social,wang2020optimal} considered leader-follower settings in which leaders observe the rewards and choices of their followers and followers copy the choices of their leaders. Paper \cite{kolla2018collaborative} also considered a decentralized bandit problem where agents observe rewards and choices of their neighbors without directly copying their choices.

The papers \cite{chakraborty2017coordinated,madhushani2020distributed,madhushani2020doesn,madhushani2020dynamic,madhushani2019heterogeneous,madhushani2020heterogeneous} considered decentralized bandit problems with time-varying communication. The papers  \cite{chakraborty2017coordinated,madhushani2020distributed,madhushani2020doesn,madhushani2020dynamic} proposed communication protocols in which agents decide when to communicate depending on the decision-making process. The paper  \cite{chakraborty2017coordinated} considered an additional constraint in which agents cannot communicate and sample at the same time.  The papers  \cite{madhushani2020heterogeneous,madhushani2019heterogeneous} proposed stochastic communication protocols. 
The problem presented in our previous paper \cite{madhushani2019heterogeneous} is closely related to the problem we study in this paper. In \cite{madhushani2019heterogeneous} neighbors are defined according to a $d$-regular graph and each agent $k$ observes each of its neighbors independently with probability $p_k.$ We provided an index, dependent on  $p_k$'s, to rank  agents according to   performance. In \cite{madhushani2020heterogeneous}  heterogeneous sampling rules were derived to improve group performance in multi-star networks.

We are motivated by real-world examples modeled by the bandit framework, such as recommender systems \cite{warlop2018fighting}, user-targeted online advertising \cite{Durand2018ContextualBF,feraud2018decentralized} and clinical trials \cite{tossou2016algorithms}, in which, to protect the privacy of users or patients, it is desirable not to disclose the history of choices. We study a distributed bandit problem in which agents share information only about the last obtained reward with their neighbors defined according to a fixed $d$-regular network graph. We consider the case where  each edge of the communication graph independently fails with probability $1-p.$ The edge failure can be either considered as an inherent imperfection of the communication links or a decision made by agents to share 
rewards at random time steps. We first propose a novel UCB based algorithm and then analyze the effect of degree $d$ and edge weight probability $p$ on group performance.

In Section \ref{Secn:MAMAB} we formulate the problem mathematically and in Section  \ref{sec:algorithm} we propose a new UCB based algorithm. We analyze the performance of the proposed algorithm and prove theoretical results in Section \ref{Secn:Regret}. In Section \ref{SubSecn:complete} we provide specialized results for the case where the communication network is a complete graph. In Section \ref{SubSecn:comp} we compare theoretical results of our algorithm and state-of-the-art algorithms. In Section \ref{Secn:Simu} we provide numerical simulations illustrating our results and validating theoretical guarantees. In Section \ref{subsec:dregsim} we show results for $d$-regular graphs. In Section \ref{SubSecn:fail} and \ref{SubSecn:NoFail} we compare our results with state-of-the-art algorithms. We provide additional discussions in Section \ref{SubSecn:discus} and concluding remarks in Section \ref{Secn:Concl}.

\section{Problem Formulation}\label{Secn:MAMAB}

In this section we present the mathematical formulation of the problem. Consider a group of $K$ agents faced with the same $N$-armed bandit problem for $T$ time steps. The terms arms and options are used interchangeably. Let $X_i$ be a sub-Gaussian random variable, with variance proxy $\sigma_i^2$, that denotes the reward of option $i\in \{1,2,\ldots,N\}.$ Sub-Gaussian distributions include Gaussian distributions, Bernoulli distributions, and bounded distributions, which are ubiquitous in real-world applications of the bandit framework.  Define $\mathbb{E}\left(X_i\right)=\mu_i$ as the expected reward of option $i.$ The option with maximum expected reward is the optimal option $i^*=\arg\max \{\mu_1,\ldots, \mu_N\}.$ Let $\Delta_i=\mu_{i^*}-\mu_i$ be the expected reward gap between option $i^*$ and option $i.$  

Let the random variable $\varphi_t^k\in \{1,2,\ldots,N\}$ denote the option chosen by agent $k$ at time $t.$ Let $\mathbb{I}_{\{\varphi_t^k=i\}}$ be the indicator random variable that takes value 1 if agent $k$ chose  option $i$ at time $t$ and 0 otherwise. 

We define the communication network using a family of time-varying graphs denoted by $\{G_t\left(\mathcal{V},\mathcal{E}_t\right)\}_{t=1}^T.$ The 
graph is generated by an underlying fixed $d$-regular graph $G(\mathcal{V},\mathcal{E})$ and by letting each edge fail independently with probability $1-p.$ Here each node $k\in\mathcal{V}$ represents an agent and each edge $e(k,j)\in\mathcal{E}$ represents the undirected communication link between agents $k$ and $j.$ Let $\mathbb{I}^t_{\{k,j\}}$ be the random variable that takes value 1 if agent $k$ observes  the reward obtained by agent $j$ at time $t$ and 0 otherwise. Thus, $\mathbb{E}\left(\mathbb{I}_{\{k,j\}}\right)=p, \forall k,j$ if $e(k,j)\in \mathcal{E}.$ Note that $\mathbb{E}\left(\mathbb{I}_{\{k,k\}}\right)=1, \forall k.$ We consider the case where the structure of the communication graph is not known to the agents. When the edge between agents $k$ and $j$ at time $t$ fails, $\mathbb{I}^t_{\{k,j\}}= \mathbb{I}^t_{\{j,k\}}=0.$ Thus, at each time step the communication graph is undirected. 

\section{Algorithm}\label{sec:algorithm}

In this section we present the mathematical formulation of our algorithm. Let $\widehat{\mu}^k_i(t)$ be the estimated mean of option $i$ by agent $k$ at time $t.$ Define $n_i^k(t)=\sum_{\tau=1}^t\mathbb{I}_{\{\varphi_{\tau}^k=i\}}$ as the number of samples of option $i$ obtained  by agent $k$ until time $t.$ Let $N_i^k(t)$ be the sum of $n_i^k(t)$ and number of samples of option $i$ communicated to agent $k$ by its neighbors until time $t.$ Then by definition  $N_i^k(t)=\sum_{\tau=1}^t\sum_{j=1}^K\mathbb{I}_{\{\varphi_{\tau}^j=i\}}\mathbb{I}^{\tau}_{\{k,j\}}.$

The estimated mean value is calculated by taking the simple average of the sum of reward values of option $i$ obtained by, and communicated to, agent $k$ up to time $t$: 
\begin{align*}
\widehat{\mu}^k_i(t)=\frac{S_i^k(t)}{N^k_i(t)}
\end{align*}
where $S_i^k(t)=\sum_{\tau=1}^t\sum_{j=1}^K X_i\mathbb{I}_{\{\varphi^j_{\tau}=i\}}\mathbb{I}^{\tau}_{\{k,j\}}.$

The goal of each agent is to maximize its individual cumulative reward while contributing to maximizing the group cumulative reward. In order to realize this goal, each agent employs an agent-based strategy that captures the trade-off between exploring and exploiting by constructing an objective function that strikes a balance between the estimation of the expected reward and the uncertainty associated with the estimate. We consider the case with unknown edge failure probability but known variance proxy as follows.

\begin{assumption}\label{asm:KnownVariance}
\normalfont
We assume that agents know the total number of  agents $K$ and the variance proxy $\sigma_i^2$ of the rewards associated with each option.
\end{assumption}

The statement regarding the variance proxy in Assumption~\ref{asm:KnownVariance} is equivalent to knowing the variance when the reward distributions are Gaussian and knowing the upper and lower bounds when the reward distributions are bounded. We make the following assumption about the rewards obtained.

\begin{assumption}
\normalfont
When more than one agent chooses the same option at the same time they receive rewards independently drawn from the probability distribution associated with the chosen option.
\end{assumption}

Our proposed algorithm is the sampling rule for each agent $k$ given in Definition \ref{def:samplerule}.

\begin{definition} {\bf{(Sampling Rule)}}\label{def:samplerule}
The sampling rule $\{\varphi^k_t\}_1^{T}$ for  agent $k$ at time $t \in \{1, \ldots, T\}$ is defined as
    \begin{align*}
    \mathbb{I}_{\{\varphi^k_{t+1}=i\}}=\left\{
    \begin{array}{cl} 1 &, \:\:\:i=\arg \max \{Q^k_1(t),\cdots,Q^k_N(t)\}\\
      0 &, \:\:\: {\mathrm{o.w.}}\end{array}\right.
    \end{align*}
    with     
    \begin{align}
    Q^k_i(t)&\triangleq \widehat{\mu}^k_{i}(t)+C^k_i(t)\label{eq:UCBQ}\\
    C^k_i(t)&\triangleq\sigma_{i}\sqrt{\frac{2\log \left(t^{\xi+1}K\right)}{N^k_{i}(t)}}\label{eq:Uncertainity}
    \end{align}
    where $\xi>1.$

\end{definition}
Note that the term $C^k_i(t)$ depends on the uncertainty associated with  agent $k'$s estimate of mean reward of option $i$. At each time step, the agent chooses the option that maximizes the objective function: the corresponding sequence of choices addresses the trade-off between exploiting and exploring. According to the sample rule defined by Definition \ref{def:samplerule}, each agent uses the information it receives from other agents, in addition to the information obtained through its own sampling, to compute the objective functions. Thus, through information sharing, agents reduce the uncertainty associated with estimates of expected reward of the options.

\section{Performance Analysis}\label{Secn:Regret}
In this section we analyze the performance of our algorithm. Recall that regret is the loss due to choosing a suboptimal option instead of the optimal option. The expected group cumulative regret can be given as
\begin{align}
\mathbb{E}\left(R(T)\right)=\sum_{i=1}^N\sum_{k=1}^K\left(\mu_{i^*}-\mu_i\right)\mathbb{E}\left(n^k_i(T)\right).\label{eq:regret}
\end{align}
Thus, the expected group cumulative regret can be minimized by minimizing the expected number of samples taken from suboptimal options.

We proceed to upper bound the expected number of samples obtained from suboptimal options as follows. Recall that $n_i^k(t)$ is the number of samples of option $i$ obtained by agent $k$ until time $t.$ Thus for each suboptimal option $i$ we have
\begin{align}
\sum_{k=1}^K\mathbb{E}\left(n_i^k(T)\right) &=\sum_{k=1}^K\sum_{t=1}^T\mathbb{P}\left(\varphi_t^k=i\right)\nonumber\\
&\leq \sum_{k=1}^K\sum_{t=1}^T\mathbb{P}\left(Q_i^k(t)> Q_{i^*}^k(t)\right).\label{eq:Expsum}
\end{align}
Let $\{\eta_i(t)\}_1^T$ be a family of nonnegative nondecreasing functions. Then from (\ref{eq:Expsum}) we have
\begin{align}
\sum_{k=1}^K\mathbb{E}\left(n_i^k(T)\right) 
\leq\sum_{k=1}^K\sum_{t=1}^T\mathbb{P}\left(\varphi_t^k=i,N_i^k(t)\leq\eta_i(t)\right)
\nonumber\\
+\sum_{k=1}^K\sum_{t=1}^T\mathbb{P}\left(Q_i^k(t)> Q_{i^*}^k(t),n_i^k(t)>\eta_i(t)\right).\label{eq:etaTrick}
\end{align}
From (\ref{eq:UCBQ}) in Definition \ref{def:samplerule} we have
\begin{align}
&\left\{Q_i^k(t)\geq Q_{i^{*}}^{k}(t)\right\}\subseteq \left\{\mu_{i^{*}}<\mu_{i}+2C_{i}^{k}(t)\right\}\nonumber\\
&\cup\left\{\widehat{\mu}_{i^{*}}^{k}(t)\leq \mu_{i^{*}}-C_{i^{*}}^{k}(t)\right\}\cup \left\{\widehat{\mu}_{i}^{k}(t)\geq \mu_{i}+C_{i}^{k}(t)\right\}.\label{eq:set}
\end{align}
Thus from (\ref{eq:Uncertainity}) when $\eta_i(t)=\frac{8\sigma_i^2}{\Delta^2_i}\log \left(t^{\xi+1}K\right)$ 
we see that
\begin{align}
  \mathbb{P}\left(\mu_{i^{*}}<\mu_{i}+2C_{i}^{k}(t), 
n_i^k(t)\geq \eta_i(t)\right)=0.\label{eq:zeroval} 
\end{align}
Thus from (\ref{eq:etaTrick})-(\ref{eq:zeroval}) we get 
\begin{align}
\sum_{k=1}^K\mathbb{E}\left(n_i^k(T)\right) \leq&\sum_{k=1}^K\sum_{t=1}^T\mathbb{P}\left(\varphi_t^k=i,N_i^k(t)\leq\eta_i(t)\right)\nonumber\\
&+\sum_{k=1}^K\sum_{t=1}^T\mathbb{P}\left(\widehat{\mu}_{i^{*}}^{k}(t)\leq \mu_{i^{*}}-C_{i^{*}}^{k}(t)\right)\nonumber\\
& +\sum_{k=1}^K\sum_{t=1}^T\mathbb{P}\left(\widehat{\mu}_{i}^{k}(t)\geq \mu_{i}+C_{i}^{k}(t)\right)\label{eq:Bound}.
\end{align}

First, we upper bound the last two summation terms on the right hand side of (\ref{eq:Bound}) using the following lemma.

\begin{lemma}\label{lem:tail}
For any $\xi>1,$ some $\zeta>1$ and for $\sigma_i>0$ in the uncertainty $C_i^k(t)$ given by \eqref{eq:Uncertainity}, we have 
\begin{align*}
    \mathbb{P}\left(\big |\widehat{\mu}_i^k(t)-{\mu}_i\big |>C_i^k(t)\right) \leq \frac{1}{\log \zeta}\frac{\log \left((d+1)t\right)}{t^{\xi+1}K}.
\end{align*}
\end{lemma}
\begin{proof}
Substituting $d_k=d$ in Lemma 1 of the paper \cite{madhushani2020drawback} we have
for any $\vartheta > 0$ 
\begin{align}
    \mathbb{P}&\left(\Big|\widehat{\mu}_i(t)-{\mu}_i\Big |\geq\sqrt{ \frac{\vartheta}{N_{i}(t)}}\right)\nonumber\\
    &\:\:\:\:\:\:\:\:\:\:\:\:\leq \frac{\log ((d+1)t)}{\log \zeta} \exp(-2\kappa\vartheta)\label{eq:tailprob}
\end{align}
where $\kappa=\frac{1}{\sigma_i^2\left(\zeta^{\frac{1}{4}}+\zeta^{-\frac{1}{4}}\right)^2}.$ 

Let $\vartheta=2\sigma_{i}^{2}\log \left(t^{\xi+1}K\right).$ Then the proof of Lemma \ref{lem:tail} follows from  (\ref{eq:Uncertainity}) and (\ref{eq:tailprob}).
\end{proof}
Second, we upper bound the first term on the right hand side of \eqref{eq:Bound}: $\sum_{k=1}^K\sum_{t=1}^T\mathbb{P}\left(\varphi_t^k=i,N_i^k(t)\leq\eta_i(t)\right)$ as follows.
Recall that  $N_i^k(t)=\sum_{\tau=1}^t\sum_{j=1}^K\mathbb{I}_{\{\varphi_{\tau}^j=i\}}\mathbb{I}^{\tau}_{\{k,j\}}.$ Since the communication structure is independent from the decision-making process we have
\begin{align}
(1+pd)\sum_{k=1}^K\mathbb{E}\left(n_i^k(t)\right)\leq \sum_{k=1}^K\mathbb{E}\left(N_i^k(t)\right). \label{eq:expNum}
\end{align}
Since $N_i^k(t)$ is a nonnegative random variable $N_i^k(t)\leq \eta_i(t)\implies \mathbb{E}\left(N_i^k(t)\right)\leq \eta_i(t).$ Since $\eta_{i}(t)$ is a nondecreasing nonnegative function,
\begin{align}
\sum_{k=1}^K\sum_{t=1}^T\mathbb{P}\left(\varphi_t^k=i,N_i^k(t)\leq\eta_i(t)\right)\leq \frac{K}{1+pd}\eta_i(T). \label{eq:condNum}
\end{align}
Now we state the main theoretical result of the paper.
\begin{theorem} \label{thm:regretBound}
  Consider a group of $K$ agents sharing information over a $d$-regular graph with an edge failure probability of $1-p$ and applying the sampling rule of Definition~\ref{def:samplerule}.  The expected group cumulative regret satisfies
\begin{align*}
\mathbb{E}\left(R(T)\right) \leq \sum_{i=1}^N\frac{K }{1+pd}\frac{8\sigma_i^2}{\Delta_i}\left((\xi+1)\log T+\log K\right)\\
+\sum_{i=1}^N\frac{2\Delta_i}{\xi^2\log \zeta}\left(\xi^2\log (d+1)+\xi\log (d+1)+1\right).
\end{align*}
\end{theorem}
\begin{proof}
From  (\ref{eq:Bound}),  (\ref{eq:condNum}) and Lemma \ref{lem:tail} we get
\begin{align}
\sum_{k=1}^K\mathbb{E}\left(n_i^k(T)\right) \leq  \frac{K}{1+pd}\eta_i(T)+\frac{2}{\log \zeta}\sum_{t=1}^T\frac{\log \left((d+1)t\right)}{t^{\xi+1}}.\label{eq:Expnumsub}
\end{align}
Note that we have
\begin{align}
\sum_{t=1}^T\frac{\log \left((d+1)t\right)}{t^{\xi+1}}\leq \log (d+1) +\int_1^T\frac{\log (d+1) t}{t^{\xi+1}}dt\nonumber\\
\leq\log (d+1)+\frac{\log (d+1)}{\xi}+\frac{1}{\xi^2}.\label{eq:inttoSUm}
\end{align}
The proof of Theorem \ref{thm:regretBound} follows from (\ref{eq:regret}), (\ref{eq:Expnumsub}) and (\ref{eq:inttoSUm}).
\end{proof}
\begin{remark}\label{rem:dregprob}
When $\Delta_i$ is small and $\sigma_i$ is large, the case where it is difficult to identify the optimal option, the term $\sum_{i=1}^N\frac{K }{1+pd}\frac{8\sigma_i^2}{\Delta_i}\left((\xi+1)\log T+\log K\right)$ is  dominant and the expected cumulative regret of the group decreases as $\frac{1}{1+pd}$ with increasing $p.$ This implies,  for any $d$-regular graph, that expected group cumulative regret  decreases monotonically with increasing $p$; equivalently, group performance  increases monotonically with decreasing edge failure probability $1-p$. Further, this implies that rate of decay of regret with increasing probability $p$ increases monotonically with increasing $d$. That is, the increasing performance with decreasing edge failure probability is accelerated with greater $d$. These observations,  illustrated in Section~\ref{Secn:Simu}, justify our algorithm as well-defined in the sense that group performance improves with reduced failure probability and increased connectivity.
\end{remark}

\subsection{Complete Communication Graph}\label{SubSecn:complete}
When the communication graph is complete $d=K-1.$ Thus, from Theorem~\ref{thm:regretBound}, expected group cumulative regret is 
\begin{align*}
\mathbb{E}\left(R(T)\right) \leq \sum_{i=1}^N\frac{K }{1+p(K-1)}\frac{8\sigma_i^2}{\Delta_i}\left((\xi+1)\log T+\log K\right)\\
+\sum_{i=1}^N\frac{2\Delta_i}{\xi^2\log \zeta}\left(\xi^2\log K+\xi\log K+1\right).
\end{align*}
When $p=1$ this matches the asymptotic regret bound $O(\log (TK))$, provided in papers \cite{martinez2019decentralized} and \cite{landgren2020distributed}.

\subsection{Comparison with Previous Work}\label{SubSecn:comp}
In \cite{madhushani2019heterogeneous} we considered a similar communication structure with degree homogeneity and probabilistic heterogeneity: the underlying communication structure is a $d$-regular graph and agent $k$ observes each of its neighbors with probability $p_k.$ 
If we let $p_k=p,\forall k,$ then that problem reduces to the problem we study in this paper. In \cite{madhushani2019heterogeneous} the asymptotic expected group cumulative regret is
\begin{align*}
\mathbb{E}(R(T))= O\left(K\left(\frac{1}{2}+\frac{1}{2\sqrt{1+p}}\right)\log T+K\log K\right). 
\end{align*}
The present algorithm  provides better theoretical guarantees since, by Theorem~\ref{thm:regretBound}, the same asymptotic  group  regret is
\begin{align*}
\mathbb{E}(R(T))=O\left(\frac{K}{1+pd}\log (TK)\right).
\end{align*}

\section{Simulation Results}\label{Secn:Simu}
In this section, we provide numerical simulation results illustrating  performance of the proposed algorithm and the validity of the theoretical claims. For all the simulations presented in this section, we consider 20 agents $(K=20)$ and 10 options ($N=10$). For all simulations, we provide results with 1000 time steps ($T=1000$) using 1000 Monte Carlo simulations with $\xi=1.01$.

For all simulations provided in Section \ref{subsec:dregsim} and \ref{SubSecn:fail} we let reward distributions be Gaussian. We consider the expected reward of the optimal option to be $\mu_{i^*}=11$ and for all sub-optimal options $i\neq i^*$ to be $\mu_{i}=10$. We let variance associated with all options $i$ be $\sigma_i^2 =1$. Since the expected reward gaps  $\Delta_i=1$, $i\neq i^*$, are equal to the variances $\sigma_i^2 = 1$, it is  challenging to distinguish the optimal option from the sub-optimal options.

\subsection{Validating Remark \ref{rem:dregprob}}\label{subsec:dregsim}
In this section we provide simulation results in the case that observations received by an agent from its neighbors are noisy. We let communication noise be a zero mean Gaussian noise with variance 1. Figure \ref{Fig:d_time} shows the expected group cumulative regret for different values of $d$  when links of the communication graph independently fail with probability 0.5. The orange solid curve corresponds to the circle graph where $d=2$. 
The blue dashed curve corresponds to 
$d=10.$ The green dotted curve corresponds to 
$d=19$, which is the complete graph. These results illustrate that the expected group cumulative regret is logarithmic in time and decreases with increasing $d.$ 

\begin{figure}[!htb]
    \centering
    \includegraphics[width=0.35\textwidth]{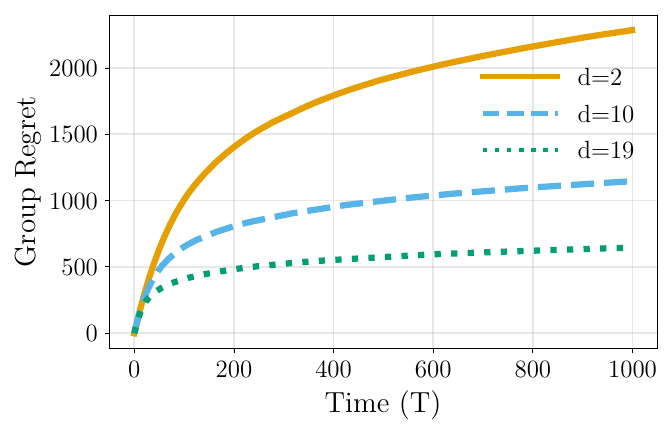}
    \caption{Expected group cumulative regret of 20 agents with $d$-regular underlying communication network and edge failure probability 0.5.}
    \label{Fig:d_time}
    \vspace{-10pt}
\end{figure}

\begin{figure}[!htb]
    \centering
    \includegraphics[width=0.35\textwidth]{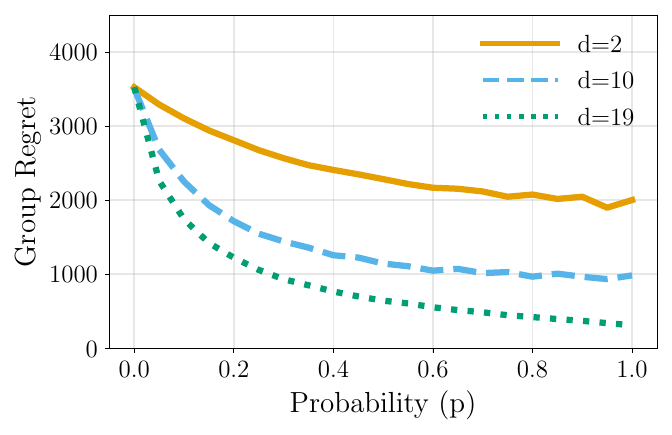}
    \caption{Expected group cumulative regret of 20 agents  with $d$-regular underlying communication network and  edge failure probability $1-p$.}
    \label{Fig:d_prob}
    \vspace{-10pt}
\end{figure}

Figure \ref{Fig:d_prob} shows how group regret changes with edge failure probability for $d=2, d=10$ and $d=19$. The expected group cumulative  regret at the end of the time horizon $T=1000$ is plotted. 
Results show that for each graph expected cumulative group regret  decreases monotonically, on average, with increasing $p$. Further notice that the rate of decay of expected group cumulative regret  increases monotonically with increasing $d.$ The results illustrate that group regret decreases with increasing $p$ as $\frac{1}{1+pd}$ thus validating the statement in Remark \ref{rem:dregprob}.

\subsection{Performance Comparison for Communication with Edge Failures}\label{SubSecn:fail}
In this section we compare  experimental results for our algorithm and the algorithm proposed in  \cite{madhushani2019heterogeneous}. We use for both algorithms the tuning parameter $\xi=1.01$ given in \cite{madhushani2019heterogeneous}. 

\begin{figure}[!htb]
    \centering
    \includegraphics[width=0.35\textwidth]{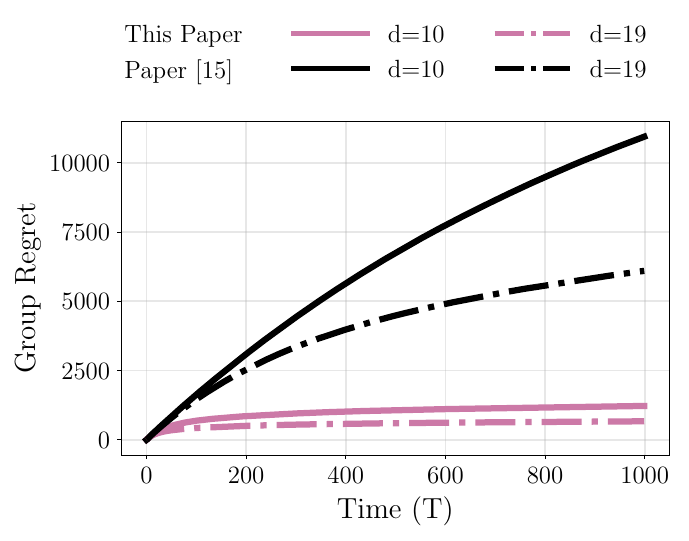}
    \caption{Expected group cumulative regret for a group of 20 agents using the algorithm of this paper and the algorithm of paper \cite{madhushani2019heterogeneous} on $d$-regular underlying communication networks ($d=10$ and $d=19$) with $p = 1-p = 0.5$ edge failure probability. }
    \label{Fig:Compare_ECC}
        \vspace{-10pt}
\end{figure}

Figure \ref{Fig:Compare_ECC} shows the expected group cumulative regret for two $d$-regular graphs,  one with $d=10$ and one with $d=19$, i.e., the complete graph, and the case in which communication links fail with probability $0.5$ ($p=0.5$). The pink solid (dotted) curve corresponds to the case in which agents use the algorithm proposed in this paper when communicating on a $d=10\: (d=19)$-regular graph. The black solid (dotted) curve corresponds to the case in which agents use the algorithm proposed in the paper \cite{madhushani2019heterogeneous} when communicating on a $d=10\:(d=19)$-regular graph.
Simulation results illustrate that the algorithm proposed in this paper performs better than the algorithm proposed in \cite{madhushani2019heterogeneous}. The performance gap between the two algorithms increases with decreasing $d$, i.e., when the number of  neighbors decreases. This illustrates that the algorithm proposed in this paper performs  significantly better  than the algorithm proposed in \cite{madhushani2019heterogeneous} when the communication graph is  sparser. 
\begin{figure}[!htb]
    \centering
     \includegraphics[width=0.35\textwidth]{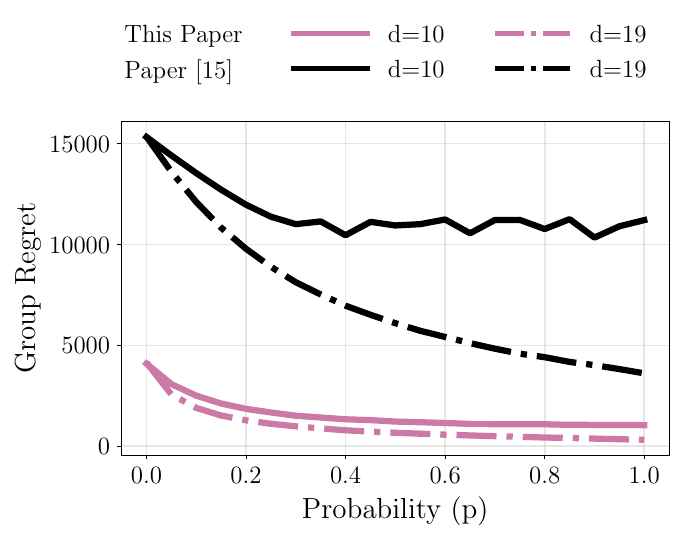}
    \caption{Expected group cumulative regret for a group of 20 agents using the algorithm of this paper and the algorithm of paper \cite{madhushani2019heterogeneous} on $d$-regular underlying communication networks ($d=10$ and $d=19$) with $1-p$ edge failure probability. Results are provided for time $T=1000$.}
    \label{Fig:Compare_ECCp}
        \vspace{-10pt}
\end{figure}

In Figure \ref{Fig:Compare_ECCp} we show that this comparative performance holds for all edge failure probability values. The figure shows how expected group cumulative regret varies with edge failure probability $1-p$. The expected group cumulative regret at the end of the time horizon $T=1000$ is plotted against probability $p$. The results illustrate that the  algorithm proposed in this paper outperforms the algorithm proposed in the paper \cite{madhushani2019heterogeneous} under all edge failure probability values.

\subsection{Performance Comparison for Communication without Edge Failures}\label{SubSecn:NoFail}
In this section we provide a comparison of results when edge failure probability is zero ($p=1$). We compare the performance of  our algorithm with the performance of the UCB-Network algorithm proposed in the paper \cite{kolla2018collaborative}. The UCB-Network algorithm does not contain any tuning parameters. We consider reward distributions to be bounded in the range $[0,1].$ Recall that bounded variables are sub-Gaussian. Let  $X\in [a,b]$ be a bounded random variable with mean $\mu$ and $b>a\geq 0.$ Then $X-\mu$ is sub-Gaussian with variance proxy $\frac{(b-a)^2}{4}.$ Thus bounded reward distributions considered in this section are sub-Gaussian with variance proxy 1/4. We consider the reward distribution of the optimal arm as a triangular distribution with mode 1 and the reward distribution of all the suboptimal arms as a triangular distribution with mode 0.

\begin{figure}[!htb]
    \centering
    \includegraphics[width=0.4\textwidth]{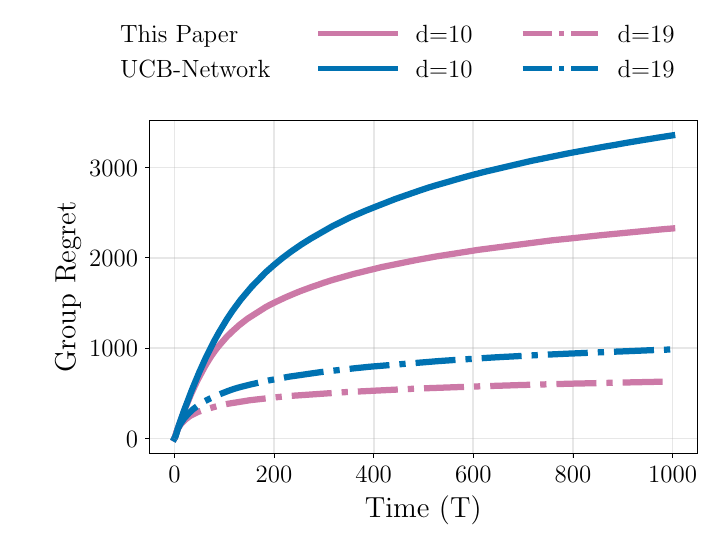}
    \caption{Expected group cumulative regret for a group of 20 agents using the algorithm of this paper and the UCB-Network algorithm  \cite{kolla2018collaborative} on $d$-regular underlying communication networks ($d=10$ and $d=19$) and no edge failures.}
    \label{Fig:Compare_kolla}
    \vspace{-10pt}
\end{figure}

Figure \ref{Fig:Compare_kolla} shows simulation results for expected group cumulative regret for $20$ agents and two $d$-regular graphs: $d=10$ and $d=19$. Results for the agents using the algorithm of this paper when communicating on a $d=10\: (d=19)$-regular graph are given in a pink solid (dotted) curve. The blue solid (dotted) curve shows results agents using the UCB-Network algorithm of \cite{kolla2018collaborative} when communicating on a $d=10\:(d=19)$-regular graph. Simulation results illustrate that our algorithm outperforms UCB-Network. The performance gap between the two algorithms increases with decreasing $d$, i.e., when the communication graph becomes  sparser our algorithm performs increasingly better than UCB-Network.

\section{Discussion}\label{SubSecn:discus}
In this section we discuss the applicability of our framework to social learning. In social learning willingness to share information with neighbors has been explained in evolutionary biology using concepts like reciprocity and intra-specific mutualism \cite{torney2011signalling}. Foraging animals use signalling as a method of communication to track food resources \cite{keasar2002bees,torney2011signalling}.  The paper \cite{keasar2002bees} considered how the bandit framework can be used to model social foraging. In this setting it is reasonable to assume that agents (animal or human) can only communicate their rewards and actions (not estimates). This is because each agent can observe rewards and action of its neighbors or each agent can signal about its rewards and actions to it neighbors. The probability $p$ can be considered as the social effort made by each agent in collective learning. Interpretation of $p$ can be such that at each time step agents observe their neighbors independently with probability $p$ or each agent broadcasts its rewards to its neighbors independently with probability $p.$ Recall that when we consider $p$ as the edge failure probability of the underlying $d$-regular graph $G,$ the resulting communication graphs $\{G_t\}_{t=1}^T$ are undirected. In the setting discussed in this section agent $k$ not observing (broadcasting) reward value and action of (to) agent $j$ at time $t$ does not necessarily mean that agent $j$ is not observing (broadcasting) reward value and action of (to) agent $k,$ and this yields a directed communication graphs $\{G_t\}_{t=1}^T$. All the results provided in this paper hold for directed communication graphs $\{G_t\}_{t=1}^T$ generated by this alternative implementation of $p$. This illustrates that our framework is applicable to a range of real-world applications from machine intelligence to evolutionary biology.

\section{Conclusions}\label{Secn:Concl}
In this paper, we studied the decentralized stochastic bandit problem with probabilistic communication failures. We proposed a new UCB based algorithm to minimize individual expected cumulative regret while contributing to minimizing expected group cumulative regret. We analyzed how agent-based strategies minimize the overall expected regret of the group. We provided improved analytical bounds for the expected group cumulative regret when the underlying communication network is a $d$-regular graph and the communication links independently fail with probability $1-p$. We proved that for any $d$-regular graph expected group cumulative regret monotonically decreases with decreasing $1-p$ (increasing $p$) as $\frac{1}{1+pd}$. We proved that when the communication structure is a complete graph with no edge failures, group regret incurred by the proposed algorithm is logarithmic in time and in number of agents. We illustrated the validity of theoretical results using numerical simulations.
 

\bibliographystyle{IEEEtran}
\bibliography{MAMAB}

\newpage
\end{document}